\def\year{2020}\relax
\documentclass[letterpaper]{article} 
\usepackage{aaai20}  
\usepackage{times}  
\usepackage{helvet} 
\usepackage{courier}  
\usepackage[hyphens]{url}  
\usepackage{graphicx} 
\usepackage{graphicx}  
\usepackage{makeidx}  
\usepackage{amsmath,bm}
\usepackage{amssymb}
\usepackage{latexsym}
\usepackage{subfigure}
\usepackage{multirow}
\usepackage{color}
\usepackage{caption}
\usepackage{marginnote}
\usepackage{lipsum}
\usepackage{changes}
\usepackage{lineno}
\usepackage{booktabs}
\usepackage{amsthm}
\usepackage{microtype}
\usepackage{threeparttable}
\usepackage{makecell}
\usepackage{indentfirst}
\usepackage[ruled,linesnumbered]{algorithm2e}

\DeclareMathOperator*{\argmin}{argmin}

\newtheorem{theorem}{Theorem}
\newtheorem{lemma}{Lemma}

\title{Random Fourier Features via Fast Surrogate Leverage Weighted Sampling}
\author{Fanghui Liu\textsuperscript{\rm 1, \rm 2}, Xiaolin Huang\textsuperscript{\rm 2, \rm 3}, Yudong Chen\textsuperscript{\rm 4}, Jie Yang\textsuperscript{\rm 2, \rm 3}, Johan A.K. Suykens\textsuperscript{\rm 1}\\
	\textsuperscript{\rm 1}Department of Electrical Engineering (ESAT-STADIUS), KU Leuven, Belgium\\
	\textsuperscript{\rm 2}Institute of Image Processing and Pattern Recognition, Shanghai Jiao Tong University, China\\
	\textsuperscript{\rm 3}Institute of Medical Robotics, Shanghai Jiao Tong University, China\\
	\textsuperscript{\rm 4}School of Operations Research and Information Engineering, Cornell University, USA\\
	fanghui.liu@kuleuven.be, xiaolinhuang@sjtu.edu.cn, yudong.chen@cornell.edu, \\ jieyang@sjtu.edu.cn, johan.suykens@esat.kuleuven.be
}
 
\begin{document}

\maketitle

\begin{abstract}
In this paper, we propose a fast surrogate leverage weighted sampling strategy to generate refined random Fourier features for kernel approximation.
Compared to the current state-of-the-art method that uses the leverage weighted scheme \cite{li2019towards}, our new strategy is simpler and more effective.
It uses kernel alignment to guide the sampling process and it can avoid the matrix inversion operator when we compute the leverage function.
Given $n$ observations and $s$ random features, our strategy can reduce the time complexity for sampling from $\mathcal{O}(ns^2+s^3)$ to $\mathcal{O}(ns^2)$, while achieving comparable (or even slightly better) prediction performance when applied to kernel ridge regression (KRR).
In addition, we provide theoretical guarantees on the generalization performance of our approach, and in particular characterize the number of random features required to achieve statistical guarantees in KRR.
Experiments on several benchmark datasets demonstrate that our algorithm achieves comparable prediction performance and takes less time cost when compared to \cite{li2019towards}.

\end{abstract}

\section{Introduction}
Kernel methods \cite{Sch2003Learning} are one of the most important and powerful tools in statistical learning. 
However, kernel methods often suffer from scalability issues in large-scale problems due to high space and time complexities.
For example, given $n$ observations in the original $d$-dimensional space $\mathcal{X}$, kernel ridge regression (KRR) requires $\mathcal{O}(n^3)$ training time and $\mathcal{O}(n^2)$ space to store the kernel matrix, which becomes intractable when $ n $ is large. 

One of the most popular approaches for scaling up kernel methods is  random Fourier features (RFF) \cite{rahimi2007random}, which approximates the original kernel by mapping input features into a new space spanned by a small number of Fourier basis.
Specifically, suppose a given kernel $k(\cdot, \cdot): \mathcal{X} \times \mathcal{X} \rightarrow \mathbb{R}$ satisfies 1) positive definiteness and 2) shift-invariance, i.e., $k(\bm x, \bm x') = k(\bm x - \bm x')$. By Bochner's theorem \cite{bochner2005harmonic}, there exists a finite Borel measure $p(\bm w)$ (the Fourier transform associated with $k$) such that
\begin{equation}\label{rffdef}
\begin{split}
  k(\bm x- \bm x') & = \int_{\mathbb{R}^d} p(\bm w) \exp\big(\mbox{i}\bm w^{\!\top}(\bm x- \bm x')\big) \mbox{d} \bm w \,.
  \end{split}
\end{equation}
(Typically, the kernel is real-valued and thus the imaginary part in Eq.~\eqref{rffdef} can be discarded.)
One can then use Monte Carlo sampling to approximate $k(\bm x , \bm x')$ by the low-dimensional kernel $\tilde{k}_p(\bm x , \bm x') = \varphi_p(\bm x)^{\!\top} \varphi_p(\bm x') $ with the explicit mapping 
\begin{equation}\label{mapping}
  \varphi_p(\bm x) \! := \! \frac{1}{\sqrt{s}}
  \big[\exp(-\mbox{i}\bm w^{\!\top}_1 \bm x), \cdots,\exp(-\mbox{i}\bm w^{\!\top}_s \bm x)]^{\!\top}\,,
\end{equation}
where $\{ \bm w_i\}_{i=1}^s$ are sampled from $p(\bm w)$ \emph{independently} of the training set.
For notational simplicity, here we write $z_p(\bm w_i, \bm x_j) := 1/\sqrt{s} \exp(-\mbox{i}\bm w^{\!\top}_i \bm x_j)$ such that $\varphi_p(\bm x)=[z_p(\bm w_1, \bm x), \cdots, z_p(\bm w_s, \bm x)]^{\!\top}$.
Note that we have $k(\bm x, \bm x') = \mathbb{E}_{\bm w \sim p} [\varphi_p( \bm x)^{\!\top} \varphi_p(\bm x')] \approx \tilde{k}_p(\bm x , \bm x') = \sum_{i=1}^{s} z_p(\bm w_i, \bm x) z_p(\bm w_i, \bm x')$.
Consequently, the original kernel matrix $\bm K = [k(\bm x_i, \bm x_j)]_{n \times n} $ on the $n$ observations $\bm X = \{  \bm x_i \}_{i=1}^n $ can be approximated by $\bm K \approx \tilde{\bm K}_p =\bm Z_p \bm Z_p^{\! \top} $, where  $\bm Z_p =[\varphi_p(\bm x_1), \cdots, \varphi_p(\bm x_n)]^{\!\top} \in \mathbb{R}^{n \times s}$.
With $ s $ random features, this approximation applied to KRR only requires $\mathcal{O}(ns^2)$ time and $\mathcal{O}(ns)$ memory, hence achieving  a substantial computational saving when $s \ll n$.

Since RFF uses the Monte Carlo estimates that are independent of the training set, a large number of random features are often required to achieve competitive approximation and generalization performance.
To improve performance, recent works \cite{sun2018but,li2019towards} consider using the ridge leverage function \cite{bach2017equivalence,avron2017random} defined with respect to the training data.
For a given random feature $\bm w_i$, this function is defined as
\begin{equation}\label{llambda}
l_{\lambda}(\bm w_i)=p(\bm w_i) \bm z^{\!\top}_{p,\bm{w}_i}(\bm{X})(\bm{K}+n \lambda \bm{I})^{-1} \bm z_{p,\bm{w}_i}(\bm{X})\,,
\end{equation}
where $\lambda$ is the regularization parameter in KRR and $\bm z_{p,\bm{w}_i}(\bm{X}) \in \mathbb{R}^n$ is the $i$-th column of $\bm Z_p$ given by $(\bm z_{p,\bm{w}_i}(\bm{X}))_j  := z_{p}(\bm{w}_i,\bm{x}_j) $.
Observe that $q^*(\bm w):= \frac{l_{\lambda}(\bm w) }{\int l_{\lambda}(\bm w) \text{d} \bm w }$ can be viewed as a probability density function,
and hence is referred to as the \emph{Empirical Ridge Leverage Score} (ERLS) distribution \cite{avron2017random}. Therefore, one can sample the features $\{ \bm w_i\}_{i=1}^s$ according to $q^*(\bm w)$, which is an importance weighted sampling strategy.
Compared to standard Monte Carlo sampling for RFF, $q^*(\bm w)$-based sampling requires fewer Fourier features and enjoys theoretical guarantees \cite{avron2017random,li2019towards}.

However, computing the ridge leverage scores and the ERLS distribution may be intractable when $n$ is large, as we need to invert the kernel matrix in Eq.~\eqref{llambda}.
An alternative way in \cite{sun2018but,li2019towards} is to use the subset of data to approximate $\bm K$, but this scheme still needs $\mathcal{O}(ns^2+s^3)$ time to generate random features.
To address these computational difficulties, we design a simple but effective leverage function to replace the original one.
For a given $\bm w$, our leverage function is defined as
\begin{equation}\label{lw}
L_{\lambda}(\bm w)=p(\bm w) \bm z^{\!\top}_{p,\bm{w}}(\bm{X}) \left( \frac{1}{n^2 \lambda} \Big(\bm y \bm y^{\!\top}+n \bm I\Big) \right) \bm z_{p,\bm{w}}(\bm{X})\,,
\end{equation}
where the matrix $\bm y \bm y^{\!\top}$ is an ideal kernel that directly fits the training data with 100\% accuracy in classification tasks, and thus can be used to guide kernel learning tasks as in kernel alignment \cite{Cortes2012Algorithms}.
It can be found that, our \emph{surrogate} function avoids the matrix inversion operator so as to further accelerate kernel approximation.
Note that, we introduce the additional term $n \bm I$ and the coefficient $1/(n^2 \lambda)$ in Eq.~\eqref{lw} to ensure, $ L_{\lambda} $ is a  \emph{surrogate} function that upper bounds $ l_{\lambda} $ in Eq.~\eqref{llambda} for theoretical guarantees.
This can be achieved due to $L_{\lambda}(\bm w_i) \ge l_{\lambda}(\bm w_i)  $\footnote{It holds by $(\bm{K} + n \lambda \bm{I})^{-1} \preceq  (n \lambda \bm{I})^{-1} \preceq \frac{1}{n^2 \lambda} (\bm y \bm y^{\!\top}+n \bm I ) $, where the notation $0 \preceq \bm A$ denotes that $\bm A$ is semi-positive definite.}.
In this way, our method with the \emph{surrogate} function requires less computational time while achieving comparable generalization performance, as demonstrated by our theoretical results and experimental validations.

Specifically, the main contributions of this paper are:
\begin{itemize}
  \item We propose a surrogate ridge leverage function based on kernel alignment and derive its associated fast surrogate ERLS distribution. This distribution is simple in form and has intuitive physical meanings.
      Our theoretical analysis provides a lower bound on the number of random features that guarantees no loss in the learning accuracy in KRR.
  \item By sampling from the surrogate ERLS distribution, our \emph{data-dependent} algorithm takes $\mathcal{O}(ns^2)$ time to generate random features, which is the same as RFF and less than the $\mathcal{O}(ns^2+s^3)$ time in \cite{li2019towards}. We further provide theoretical guarantees on the generalization performance of our algorithm equipped with the KRR estimator.
  \item Experiments on various benchmark datasets demonstrate that our method performs better than standard random Fourier features based algorithms. Specifically, our algorithm achieves comparable (or even better) accuracy and uses less time when compared to \cite{li2019towards}.
\end{itemize}
The remainder of the paper is organized as follows. Section~\ref{sec:pre} briefly reviews the related work on random features for kernel approximation.
Our surrogate leverage weighted sampling strategy for RFF is presented in Section \ref{sec:model}, and related theoretical results are given in Section~\ref{sec:theo}.
In section \ref{sec:exp}, we provide experimental evaluation for our algorithm and compare with other representative random features based methods on popular benchmarks.
The paper is concluded in Section \ref{sec:conclusion}.

\section{Related Works}
\label{sec:pre}

Recent research on random Fourier features focuses on constructing the mapping
\begin{equation*}\label{phiweight}
   \varphi(\bm x)  := \frac{1}{\sqrt{s}}
   \big[a_1 \exp(-\mbox{i}\bm w^{\!\top}_1 \bm x), \cdots, a_s\exp(-\mbox{i}\bm w^{\!\top}_s \bm x)]^{\!\top}\,.
\end{equation*}
The key question is how to select the points $\bm w_i$ and weights $a_i$ so as to uniformly approximate the integral in Eq.~\eqref{rffdef}.
In standard RFF, $\{ \bm w_i \}_{i=1}^s$ are randomly sampled from $p(\bm w)$ and the weights are equal, i.e., $a_i \equiv  1$.
To reduce the approximation variance, \citeauthor{Yu2016Orthogonal} \shortcite{Yu2016Orthogonal} proposes the orthogonal random features (ORF) approach, which incorporates an orthogonality constraint  when sampling $\{ \bm w_i \}_{i=1}^s $ from $p(\bm w)$.
Sampling theory \citeauthor{niederreiter1992random} \shortcite{niederreiter1992random} suggests that the convergence of Monte-Carlo used in RFF and ORF can be significantly improved by choosing a deterministic sequence $\{ \bm w_i \} $ instead of sampling randomly.
Therefore, a possible middle-ground method is Quasi-Monte Carlo sampling \cite{Avron2016Quasi},  which uses a low-discrepancy sequence $\{ \bm w_i \}$ rather than the fully random Monte Carlo samples.
Other deterministic approaches based on numerical quadrature are considered in \cite{evans1993practical}.
\citeauthor{bach2017equivalence} \shortcite{bach2017equivalence} analyzes the relationship between random features and quadrature, which allows one to use deterministic numerical integration methods such as Gaussian quadrature \cite{dao2017gaussian}, spherical-radial quadrature rules \cite{munkhoeva2018quadrature}, and sparse quadratures \cite{GauthierSIAM2018} for kernel approximation.

The above methods are all \emph{data-independent}, i.e., the selection of points and weights is independent of the training data.
Another line of work considers \emph{data-dependent} algorithms, which use the training data to guide the generation of random Fourier features by using, e.g., kernel alignment \cite{AmanNIPS2016}, feature compression \cite{agrawal2019data}, or the ridge leverage function \cite{avron2017random,sun2018but,li2019towards,fanuel2019nystr}.
Since our method builds on the leverage function $l_{\lambda}(\bm w)$,  we detail this approach here.
From Eq.~\eqref{llambda}, the integral of $l_{\lambda}(\bm w)$ is
\begin{equation}\label{dklambda}
  \int_{\mathbb{R}^d} l_{\lambda}(\bm w) \mbox{d} \bm w = \operatorname{Tr}\left[\bm{K}(\bm{K}+n \lambda \bm{I})^{-1}\right] =: d_{\bm{K}}^{\lambda}\,.
\end{equation}
The quantity $d_{\bm{K}}^{\lambda} \ll n$ determines the number of independent parameters in a learning problem and hence is referred to as the \emph{number of effective degrees of freedom} \cite{bach2013sharp}.
\citeauthor{li2019towards} \shortcite{li2019towards} provides the sharpest bound on the required number of random features; in particular, with $\Omega(\sqrt{n} \log d_{\bm{K}}^{\lambda})$ features, no loss is incurred in  learning accuracy of kernel ridge regression.
Albeit elegant, sampling according to $q^*(\bm w)$ is often intractable in practice. The alternative approach proposed in \cite{li2019towards} takes $\mathcal{O}(ns^2+s^3)$ time, which is larger than $\mathcal{O}(ns^2)$ in the standard RFF.

\section{Surrogate Leverage Weighted RFF}
\label{sec:model}

\subsection{Problem setting}

Consider a standard supervised learning setting, where $\mathcal{X}$ is a compact metric space of features, and $\mathcal{Y} \subseteq \mathbb{R}$ (in regression) or $\mathcal{Y}=\{ -1, 1\}$ (in classification) is the label space. We assume that a sample set $\{  (\bm x_i, y_i) \}_{i=1}^n $ is drawn from a non-degenerate Borel probability measure $\rho$ on $\mathcal{X} \times \mathcal{Y}$.
The \emph{target function} of $\rho$ is defined by
$f_{\rho}(\bm x) := \int_\mathcal{Y} y \mathrm{d} \rho(y|\bm x)$ for each $\bm x \in \mathcal{X}$, where $\rho(\cdot|\bm x)$ is the conditional distribution of $\rho$ at $\bm x $.
Given a kernel function $k$ and its associated reproducing kernel Hilbert space (RKHS) $\mathcal{H}$, the goal is to find a hypothesis 
$f: \mathcal{X} \rightarrow \mathcal{Y}$ in $\mathcal{H}$ such that  $f(\bm x)$ is a good estimate of the label $y \in \mathcal{Y}$ for a new instance $\bm x \in \mathcal{X}$.
By virtue of the representer theorem \cite{Sch2003Learning}, an empirical risk minimization problem can be formulated as
\begin{equation}\label{ermp}
\hat{f}^{\lambda} :=\underset{f \in \mathcal{H}}{\argmin }~ \frac{1}{n} \sum_{i=1}^{n} \ell \left(y_{i},f(\bm x_i)\right)+\lambda \bm \alpha^{\!\top} \bm{K} \bm \alpha\,,
\end{equation}
where $f = \sum_{i=1}^{n} \alpha_i k(\bm x_i, \cdot)$ with $\bm \alpha \in \mathbb{R}^n$ and the convex loss $\ell: \mathcal{Y} \times \mathcal{Y} \rightarrow \mathbb{R}$ quantifies the quality of the estimate $f(\bm x)$ w.r.t.\ the true $y $.
In this paper, we focus on learning with the squared loss, i.e., $\ell(y, f(\bm x)) = (y - f(\bm x))^2$.
Hence, the expected risk in KRR is defined as $\mathcal{E}(f) = \int_{\mathcal{X} \times \mathcal{Y}} (f(\bm x) - y)^2 \mathrm{d} \rho$, with the corresponding empirical risk defined on the sample, i.e., $\hat{\mathcal{E}}(f) = \frac{1}{n} \sum_{i=1}^{n} \big(f( \bm x_i) - y_i \big)^2$.
In standard learning theory, the optimal hypothesis $f_{\rho}$ is defined as $f_{\rho}(\bm x) = \int_Y y \mathrm{d} \rho(y|\bm x), \quad \bm x \in \mathcal{X}$,
where $\rho(\cdot|\bm x)$ is the conditional distribution of $\rho$ at $\bm x \in \mathcal{X}$. The regularization parameter $\lambda$ in Eq.~\eqref{ermp} should depend on the sample size; in particular, $\lambda \equiv \lambda(n)$ with $\lim_{n \rightarrow \infty} \lambda(n) = 0$.
Following \cite{Rudi2017Generalization,li2019towards}, we pick $\lambda \in \mathcal{O}(n^{-1/2})$.

As shown in \cite{li2019towards}, when using random features, the empirical risk minimization problem~\eqref{ermp} can be expressed as
\begin{equation}\label{ermrff}
\bm \beta_{\lambda} :=\underset{\bm \beta \in \mathbb{R}^{s}}{\argmin } ~\frac{1}{n}\left\|\bm y- \bm {Z}_{q} \bm \beta\right\|_{2}^{2}+\lambda \|\bm \beta\|_{2}^{2}\,,
\end{equation}
where $\bm y = [y_1, y_2, \cdots, y_n]^{\!\top}$ is the label vector and  $\bm Z_q =[\varphi_q(\bm x_1), \cdots, \varphi_q(\bm x_n)]^{\!\top} \in \mathbb{R}^{n \times s}$ is the random feature matrix, with $ \varphi_q(\bm x) $ as defined in Eq.~\eqref{mapping} and $\{ \bm w_i \}_{i=1}^s$ sampled from a distribution $q(\bm w)$.
Eq.~\eqref{ermrff} is a linear ridge regression problem in the space spanned by the random features \cite{suykens2002least,mall2015very}, and the optimal hypothesis is given by $\bm f_{\bm \beta}^{\lambda} = \bm Z_q \bm \beta_{\lambda}$, with
\begin{equation}\label{dualpot}
  \bm \beta_{\lambda} = (\bm Z_q^{\!\top} \bm Z_q + n \lambda \bm I)^{-1} \bm Z_q^{\!\top} \bm y\,.
\end{equation}

Note that the distribution $q(\bm w)$ determines the feature mapping matrix and hence has a significant impact on the generalization performance.
Our main goal in the sequel is to design a good $q(\bm w)$, and to understand the relationship between $q(\bm w)$ and the expected risk.
In particular, we would like to characterize the number $ s $ of random features needed when sampling from $q(\bm w)$  in order to achieve a certain convergence rate of the risk.

\subsection{Surrogate leverage weighted sampling}
Let $q(\bm w)$ be a probability density function to be designed. Given the points $\{ \bm w_i \}_{i=1}^s$ sampled from $q(\bm w)$, we define the mapping
\begin{equation}\label{fmaping}
\varphi_q(\bm x) = \frac{1}{\sqrt{s}} \left(\sqrt{\frac{p\left(\bm w_1\right)}{q\left(\bm w_1\right)}}  e^{-\mbox{i}\bm w^{\!\top}_1 \bm x}, \cdots, \sqrt{\frac{p\left(\bm w_s\right)}{q\left(\bm w_s\right)}}  e^{-\mbox{i}\bm w^{\!\top}_s \bm x}\right)^{\!\!\top}\,.
\end{equation}
We again have $k(\bm x, \bm x')=\mathbb{E}_{\bm{w} \sim q } [\varphi_q( \bm x)^{\!\top} \varphi_q(\bm x')] \approx \tilde{k}_q(\bm x, \bm x') = \sum_{i=1}^{s} z_q(\bm w_i, \bm x) z_q(\bm w_i, \bm x')$, where $z_q(\bm w_i, \bm x_j) := \sqrt{p(\bm w_i)/ q(\bm w_i)} z_p(\bm w_i, \bm x_j)$.
Accordingly, the kernel matrix $\bm K$ can be approximated by $\bm K_q = \bm Z_q \bm Z_q^{\!\top}$, where $\bm Z_q :=[\varphi_q(\bm x_1), \cdots, \varphi_q(\bm x_n)]^{\!\top} \in \mathbb{R}^{n \times s}$.
Denoting by $\bm z_{q,\bm w_i}(\bm X)$  the $i$-th column of $\bm Z_q$,
we have $\bm K = \mathbb{E}_{\bm w \sim p}[\bm z_{p,\bm w}(\bm X) \bm z^{\!\top}_{p,\bm w}(\bm X)]=\mathbb{E}_{\bm w \sim q}[\bm z_{q,\bm w}(\bm X) \bm z^{\!\top}_{q,\bm w}(\bm X)]$.
Note that this scheme can be regarded as a form of importance sampling.

Our surrogate empirical ridge leverage score distribution $L_{\lambda}(\bm w)$ is given by Eq.~\eqref{lw}.
The integral of $L_{\lambda}(\bm w)$ is
\begin{equation}\label{lwour}
   \int_{\mathbb{R}^d} L_{\lambda}(\bm w) \mbox{d} \bm w = \frac{1}{n^2 \lambda}\operatorname{Tr}\left[\big(\bm y \bm y^{\!\top} \!\! + n \bm I \big) \bm{K}\right] :=D^{\lambda}_{\bm{K}}\,.
\end{equation}
Combining Eq.~\eqref{lw} and Eq.~\eqref{lwour}, we can compute the surrogate empirical ridge leverage score distribution by
\begin{equation}\label{serls}
  q(\bm w) := \frac{L_{\lambda}(\bm w) }{\int_{\mathbb{R}^d} L_{\lambda}(\bm w) \text{d} \bm w } = \frac{L_{\lambda}(\bm w)}{D^{\lambda}_{\bm{K}}}\,.
\end{equation}
The random features $\{ \bm w_i \}_{i=1}^s$ can then be sampled from the above $q(\bm w)$. We refer to this sampling strategy as \emph{surrogate leverage weighted RFF}.
Compared to the standard $ l_{\lambda} $ and its associated ERLS distribution, the proposed $L_{\lambda}(\bm w)$ and $D^{\lambda}_{\bm{K}}$ are simpler: it does not require inverting the kernel matrix and thus accelerates the generation of random features.

Since the distribution $q(\bm w)$ involves the kernel matrix $\bm K$ that is defined on the entire training dataset, we need to approximate $ \bm K $ by random features, and then calculate/approximate $ q(\bm w) $.
To be specific, we firstly sample $\{ \bm w_i \}_{i=1}^l$ with $l \geq s$ from the spectral measure $p(\bm w)$ and form the feature matrix $\bm Z_l \in \mathbb{R}^{n \times l}$.
We have $\bm K \approx  \tilde{\bm K}_p = \bm Z_l \bm Z_l^{\!\top}$, and thus the distribution $q(\bm w)$ can be approximated by
\begin{equation}\label{qtildewnew}
  \tilde{q}(\bm w) = \frac{p(\bm w) \bm z^{\!\top}_{p,\bm{w}}(\bm{X}) \left( \bm y \bm y^{\!\top}+n \bm I \right) \bm z_{p,\bm{w}}(\bm{X})}{ \| \bm y^{\!\top} \! \bm Z_l \|_2^2 + n \| \bm Z_l \|^2_{\text{F}}  }\,.
\end{equation}
Hence, we can then sample from $\tilde{q}(\bm w)$ to generate the refined random features by importance sampling.

Note that the term $n \bm I$ in Eq.~\eqref{lw} and Eq.~\eqref{qtildewnew} is independent of the sample set $\bm X$. 
If we discard this term in our algorithm implementation, $L_{\lambda}(\bm w)$ in Eq.~\eqref{lw} can be transformed as
\begin{equation}\label{lwapp}
  L'_{\lambda}(\bm w)=p(\bm w) \bm z^{\!\top}_{p,\bm{w}}(\bm{X}) \left( \frac{1}{n^2 \lambda} \bm y \bm y^{\!\top} \right) \bm z_{p,\bm{w}}(\bm{X})\,,
\end{equation}
 and further $\tilde{q}(\bm w)$ in Eq.~\eqref{qtildewnew} can be simplified to
\begin{equation}\label{qtildew}
  \tilde{q}'(\bm w) = \frac{p(\bm w) \bm z^{\!\top}_{p,\bm{w}}(\bm{X}) \left( \bm y \bm y^{\!\top} \right) \bm z_{p,\bm{w}}(\bm{X})}{ \| \bm y^{\!\top} \! \bm Z_l \|_2^2   }\,.
\end{equation}
For each feature $\bm w_i \sim \tilde{q}'(\bm w)$, its re-sampling probability $p_i$ is associated with the approximate empirical ridge leverage score in Eq.~\eqref{lwapp}.
To be specific, it can be represented as
\begin{equation}\label{samplingp}
p_i \propto \left(\bm y^{\!\top} (\bm Z_l)_i\right)^2 = \bigg|\sum_{j=1}^{n} y_j z_{p}(\bm{w}_i,\bm{x}_j)  \bigg|^2 \,, \quad i=1,2,\cdots,l\,.
\end{equation}
It has intuitive physical meanings.
From Eq.~\eqref{samplingp}, it measures the correlation between the label $y_j$ and the mapping output $z_{p}(\bm{w}_i,\bm{x}_j)$.
Therefore, $ p_i $ quantifies the contribution of $\bm w_i$, which defines the $i$-th dimension of the feature mapping $\varphi(\cdot)$, for fitting the training data.
In this view, if $ p_i $ is large, $\bm w_i$ is more important than the other features, and will be given the priority in the above importance sampling scheme.
Based on this, we re-sample $s$ features from $\{ \bm w \}_{i=1}^l$ to generate the refined random features. 
Our surrogate leverage weighted RFF algorithm applied to KRR is summarized in Algorithm~\ref{ago3}.

Also note that if the following condition holds
\begin{equation*}
  \frac{\bm y^{\!\top} \! \sum_{j=1}^{n}(\bm Z_s)_j (\bm Z_s)^{\!\top}_j \bm y}{\bm y^{\!\top} (\bm Z_s)_i (\bm Z_s)^{\!\top}_i \bm y} \! \approx \! \frac{\sum_{j=1}^{n} \| (\bm Z_s)_j \|_2^2}{\| (\bm Z_s)_i \|_2^2}\,,
\end{equation*}
then sampling from $\tilde{q}(\bm w)$ or $\tilde{q}'(\bm w)$ does not have distinct differences.
This condition is satisfied if $\| (\bm Z_s)_i \|_2$ does not dramatically fluctuate for each column.
 in which sampling from $\tilde{q}(\bm w)$ or $\tilde{q}'(\bm w)$ may be used. 

\begin{algorithm}[t]
\caption{The Surrogate Leverage Weighted RFF Algorithm in KRR}
\label{ago3}
\KwIn{the training data $\{  (\bm x_i, y_i) \}_{i=1}^n $, the shift-invariant kernel $k$, the number of random features $s$, and the regularization parameter $\lambda$}
\KwOut{the random feature mapping $\varphi(\cdot)$ and the optimization variable $\bm \beta_{\lambda}$ in KRR}
Sample random features $\{ \bm w_i \}_{i=1}^l$ from $p(\bm w)$ with $l \geq s$, and form the feature matrix $\bm Z_l \in \mathbb{R}^{n \times l}$. \\
associate with each feature $\bm w_i$ a real number $p_i$ such that $p_i$ is proportional to
\begin{equation*}
	p_i \propto \left(\bm y^{\!\top} (\bm Z_l)_i\right)^2 \,, \quad i=1,2,\cdots,l\,.
\end{equation*}
\\
Re-sample $s$ features from $\{ \bm w_i \}_{i=1}^l$ using the multinomial distribution given by the vector $(p_1/L, p_2/L, \cdots, p_l/L)$ with $L \leftarrow \sum_{i=1}^{l} p_i$.\\
Create the feature mapping $\varphi(\bm x)$ for an example $\bm x$ by Eq.~\eqref{fmaping}.\\
Obtain $\bm \beta_{\lambda}$ solved by Eq.~\eqref{dualpot}.\\
\end{algorithm}

The method in \cite{li2019towards} samples $\{ \bm w_i \}_{i=1}^s$ from $q^*(\bm w):=l_{\lambda}(\bm w) / d_{\bm{K}}^{\lambda}$, while ours samples from $q(\bm w):= L_{\lambda}(\bm w) / D_{\bm{K}}^{\lambda}$.
In comparison, our surrogate ERLS distribution is much simpler as it avoids inverting the matrix $\bm Z_s^{\!\top} \bm Z_s$.
Hence, generating $ s $ random features by Algorithm~\ref{ago3} takes $\mathcal{O}(ns^2)$ time to do the sampling.
It is the same as the standard RFF and less than the $\mathcal{O}(ns^2+s^3)$ time needed by \cite{li2019towards} which requires $\mathcal{O}(ns^2)$ for the multiplication of $\bm Z_s^{\!\top} \bm Z_s$ and  $\mathcal{O}(s^3)$ for inverting $\bm Z_s^{\!\top} \bm Z_s$.

\section{Theoretical Analysis}
\label{sec:theo}
In this section, we analyze the generalization properties of kernel ridge regression when using random Fourier features sampled from our $q(\bm w)$.
Our analysis includes two parts. We first study how many features sampled from $q(\bm w)$ are needed to incur no loss of learning accuracy in KRR.
We then characterize the convergence rate of the expected risk of KRR when combined with Algorithm~\ref{ago3}.
Our proofs follow the framework in \cite{li2019towards} and in particular involve the same set of assumptions.

\subsection{Expected risk for sampling from $q(\mathbf{w})$}

The theorem below characterizes the relationship between the expected risk in KRR and the total number of random features used.

\begin{theorem}\label{maintheo}
Given a shift-invariant and positive definite kernel function $k$, denote the eigenvalues of the kernel matrix $\bm K$ as $\lambda_1 \geq \cdots \geq \lambda_n$.
Suppose that the regularization parameter $\lambda$ satisfies $0 \leq n \lambda \leq \lambda_1$, $|y| \leq y_0$ is bounded with $y_0 > 0$, and $\{ \bm w_i \}_{i=1}^s$ are sampled independently from the surrogate empirical ridge leverage score distribution $q(\bm w) = {L_{\lambda}(\bm w)}/{D^{\lambda}_{\bm{K}}}$. If the unit ball of $\mathcal{H}$ contains the optimal hypothesis $f_{\rho}$ and
\begin{equation*}
s \geq 5 D^{\lambda}_{\bm{K}} \log \left(16 d_{\bm{K}}^{\lambda}\right) / \delta\,,
\end{equation*}
then for $0 < \delta <1$, with probability $1-\delta$, the excess risk of $f_{\bm \beta}^\lambda$ can be upper bounded as
\begin{equation}
\mathcal{E}\left(f_{\bm \beta}^{\lambda}\right)-\mathcal{E}\left(f_{\rho}\right) \leq 2 \lambda+\mathcal{O}(1 / \sqrt{n})+\mathcal{E}\big(\hat{f}^{\lambda}\big)-\mathcal{E}\left(f_{\rho}\right) \,,
\end{equation}
where $\mathcal{E}\big(\hat{f}^{\lambda}\big)-\mathcal{E}\left(f_{\rho}\right) $ is the excess risk for the standard kernel ridge regression estimator.
\end{theorem}
 Theorem~\ref{maintheo} shows that if the total number of random features  sampled from $q(\bm w)$ satisfies $s\geq 5 D^{\lambda}_{\bm{K}} \log \left(16 d_{\bm{K}}^{\lambda}\right) / \delta$, we incur no loss in the learning accuracy of kernel ridge regression.
In particular, with the standard choice $\lambda = \mathcal{O}(n^{-1/2})$, the estimator $f_{\bm \beta}^{\lambda}$ attains the minimax rate of kernel ridge regression.

To illustrate the lower bound in Theorem~\ref{maintheo} on the number of features, we consider three cases regarding the eigenvalue decay of $\bm K$: i) the exponential decay $\lambda_i \propto n e^{-ci}$ with $c > 0$, ii) the polynomial decay $\lambda_i \propto n i^{-2t}$ with $t \geq 1$, and iii) the slowest decay with $\lambda_i \propto n /i$ (see \cite{bach2013sharp} for details).
In all three cases, direct calculation shows
\begin{equation*}
  D^{\lambda}_{\bm{K}} = \frac{1}{n^2 \lambda}\operatorname{Tr}\left[ (\bm y \bm y^{\!\top} + n\bm I) \bm{K}\right] \leq \frac{2}{n \lambda} \operatorname{Tr}(\bm K) \in \mathcal{O}(\sqrt{n})\,.
\end{equation*}
Moreover, $d^{\lambda}_{\bm{K}}$ satisfies $d^{\lambda}_{\bm{K}} \in \mathcal{O}(\log n)$ in the exponential decay case, $d_{\bm K}^{\lambda} \in \mathcal{O}(n^{1/(4t)})$ in the polynomial decay case, and $d_{\bm K}^{\lambda} \in \mathcal{O}(\sqrt{n})$ in the slowest case. Combining these bounds gives the number $ s $ of random features sufficient for no loss in the learning accuracy of KRR; these results are reported in Tab.~\ref{tab:freq}.
It can be seen that sampling from $q^*(\bm w)$ \cite{li2019towards} sometimes requires fewer random features than our method.
This is actually reasonable as the design of our surrogate ERLS distribution follows in a simple fashion and we directly relax $D_K^{\lambda}$ to $\mathcal{O}(\sqrt{n})$.
It does not strictly follow with the continuous generalization of the leverage scores used in the analysis of linear methods \cite{alaoui2015fast,cohen2017input,avron2017random}.
Actually, with a more careful argument, this bound can be further improved and made tight, which we leave to future works.
Nevertheless, our theoretical analysis actually provides the worst case estimation for the lower bound of $s$. In practical uses, our algorithm would not require the considerable number of random features to achieve a good prediction performance.
Specifically, our experimental results in Section~\ref{sec:exp} demonstrate that when using the same $s$, there is no distinct difference between \cite{li2019towards} and our method in terms of prediction performance.
But our approach costs less time to generate the refined random features, achieving a substantial computational saving when the total number of random features is relatively large.

\begin{table}
\centering
\small
  \caption{
  Comparisons of the number $ s $ of features required by two sampling schemes.
  }
  \label{tab:freq}
  \begin{tabular}{ccccccc}
    \toprule[1.5pt]
  Eigenvalue decay  &\cite{li2019towards} 	&Ours \\
  \midrule[1pt]
  $\lambda_i \propto n e^{-ci}$, $c > 0$
  &$s \geq \log^2 n$
 &$s \geq \sqrt{n} \log \log n$ \\
 \hline
  $\lambda_i \propto n i^{-2t}$, $t \geq 1$
  &$s \geq n^{1/(4t)} \log n$
 &$s \geq \sqrt{n} \log n$ \\
 \hline
 $\lambda_i \propto n /i$
  &$s \geq \sqrt{n} \log n$
 &$s \geq \sqrt{n} \log n$ \\
  \bottomrule[1.5pt]
\end{tabular}
\end{table}

To prove Theorem~\ref{maintheo}, we need the following lemma.

\begin{lemma}\label{lemdk}
  Under the same assumptions from Theorem~\ref{maintheo}, let $\epsilon \geq \sqrt{m / s}+2 L / 3 s$ with constants $m$ and $L$ given by
  \begin{equation*}
  \begin{split}
    m:= D^{\lambda}_{\bm{K}} \frac{\lambda_1}{\lambda_1 + n \lambda} ~~~L:= \sup_{i} \frac{l_{\lambda}(\bm w_i)}{q(\bm w_i)},~\forall i =1,2, \cdots, s \,.
    \end{split}
  \end{equation*}
  If the number of random features satisfies
  \begin{equation}
s \geq D^{\lambda}_{\bm{K}} \left(\frac{1}{\epsilon^{2}}+\frac{2}{3 \epsilon}\right) \log \frac{16 d_{\bm{K}}^{\lambda}}{\delta}\,,
\end{equation}
then for $0 < \delta <1$, with probability $1-\delta$, we have
\begin{equation}
-\epsilon \bm{I} \preceq(\bm{K}+n \lambda \bm{I})^{-\frac{1}{2}}(\tilde{\bm{K}}_q-\bm{K})(\bm{K}+n \lambda \bm{I})^{-\frac{1}{2}} \preceq \epsilon \bm{I}\,.
\end{equation}
\end{lemma}
\begin{proof}
	Following the proof of Lemma 4 in \cite{li2019towards}, by the matrix Bernstein concentration inequality \cite{tropp2015introduction} and $l_{\lambda}(\bm w) \leq L_{\lambda}(\bm w)$, we conclude the proof.
\end{proof}
Based on Lemma~\ref{lemdk}, as well as the previous results including Lemma~2, Lemma~5, Lemma~6, Theorem~5 in \cite{li2019towards}, we can immediately prove Theorem~\ref{maintheo}.

\subsection{Expected risk for Algorithm~\ref{ago3} }
In the above analysis, our results are based on the random features $\{ \bm w_i \}_{i=1}^s$ sampled from ${q}(\bm w)$. In Algorithm~\ref{ago3}, $\{ \bm w_i \}_{i=1}^s$ are actually drawn from $\tilde{q}(\bm w)$ or $\tilde{q}'(\bm w)$.
In this section, we present the convergence rates for the expected risk of Algorithm~\ref{ago3}.
\begin{theorem}\label{maintheoapp}
Under the same assumptions from Theorem~\ref{maintheo}, denote by $\tilde{f}^{\lambda^*}$ the KRR estimator obtained using a regularization parameter $\lambda^*$ and the features $\{ \bm w_i \}_{i=1}^s$ sampled via Algorithm~\ref{ago3}. If the number of random features satisfies
	\begin{equation*}
	s \geq \max \left\{ \frac{7 z_0^2 \log \left(16 d_{\bm{K}}^{\lambda}\right) }{\lambda \delta},  5 D^{\lambda^{\!*}}_{\bm{K}} \frac{\log \left(16 d_{\bm{K}}^{\lambda^*}\right) }{ \delta} \right\} \,,
	\end{equation*}
	with $|z_p(\bm w, \bm x)|<z_0$,
	then for $0 < \delta <1$, with probability  $1-\delta$, the excess risk of $\tilde{f}^{\lambda^*}$ can be estimated by
	\begin{equation}
	\mathcal{E}(\tilde{f}^{\lambda^*}) - \mathcal{E}\left(f_{\rho}\right) \leq 2 \lambda + 2 \lambda^* +\mathcal{O}(1 / \sqrt{n})\,.
	\end{equation}
\end{theorem}
\begin{proof}
	According to Theorem~1 and Corollary~2 in \cite{li2019towards}, if the number of random features satisfies
$ s \geq {7 z_0^2 \log \left(16 d_{\bm{K}}^{\lambda}\right) }/(\lambda \delta)$, 
	then for any $0 < \delta <1$, with confidence $1-\delta$, the excess risk of $f^{\lambda}_{\bm \alpha}$ can be bounded by
	\begin{equation}\label{frisklam}
	\mathcal{E}({f}^{\lambda}_{\bm \alpha}) - \mathcal{E}\left(f_{\rho}\right) \leq 2 \lambda  +\mathcal{O}(1 / \sqrt{n})\,.
	\end{equation}
	Let $f_{\tilde{\mathcal{H}}}$ be the function in $\tilde{\mathcal{H}}$ spanned by the approximated kernel that achieves the minimal risk, i.e., $\mathcal{E}(f_{\tilde{\mathcal{H}}}) = \inf_{f \in \tilde{\mathcal{H}}} \mathcal{E}(f)$.
	Hence, we re-sample $\{\bm w_i\}_{i=1}^s$ according to $q(\bm w)$ as defined in Eq.~\eqref{serls}, in which the kernel matrix is indicated by the actual kernel $\tilde{k}$ spanned in $\tilde{\mathcal{H}}$.
	Denote our KRR estimator with the regularization parameter $\lambda^*$ and the learning function $\tilde{f}^{\lambda^*}$, and according to Theorem~\ref{maintheo}, if the number of random features $s$ satisfies
	$s \geq 5 D^{\lambda^*}_{\bm{K}} \frac{\log \left(16 d_{\bm{K}}^{\lambda^*}\right) }{ \delta}\,,$
	then for $0 < \delta <1$, with confidence $1-\delta$, the excess risk of $\tilde{f}^{\lambda^{*}}$ can be estimated by
	\begin{equation}\label{frisklams}
	\mathcal{E}\big(\tilde{f}^{\lambda^{*}}\big) - \mathcal{E}\left(f_{\tilde{\mathcal{H}}} \right) \leq 2 \lambda^{*}+\mathcal{O}(1 / \sqrt{n}) \,.
	\end{equation}
	Since $f_{\tilde{\mathcal{H}}}$ achieves the minimal risk over $\mathcal{H}$, we can conclude that $\mathcal{E}(f_{\tilde{\mathcal{H}}}) \leq \mathcal{E}(f^{\lambda}_{\bm \alpha})$.
	Combining Eq.~\eqref{frisklam} and Eq.~\eqref{frisklams}, we obtain the final excess risk of $\mathcal{E}(\tilde{f}^{\lambda^{*}})$.
\end{proof}
Theorem~\ref{maintheoapp} provides the upper bound of the expected risk in KRR estimator over random features generated by Algorithm~\ref{ago3}.
Note that, in our implementation, the number of random features used to approximate the kernel matrix is also set to $s$ for simplicity, which shares the similar way with the implementation in \cite{li2019towards}.
\begin{figure*}
\centering
\subfigure[Relative error]{\label{sapp}
\includegraphics[width=0.3\textwidth]{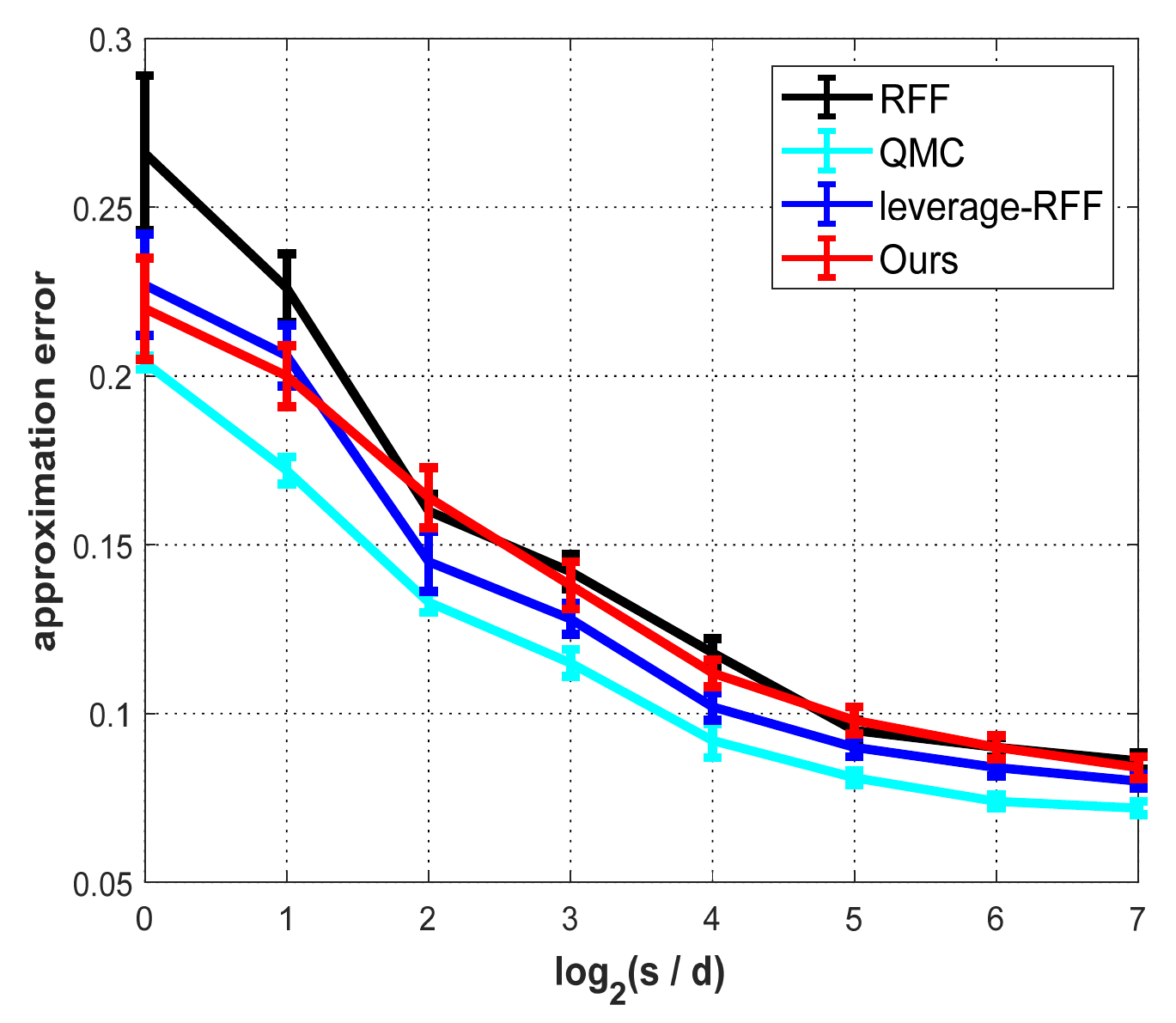}}
\subfigure[Test accuracy]{\label{sacc}
\includegraphics[width=0.3\textwidth]{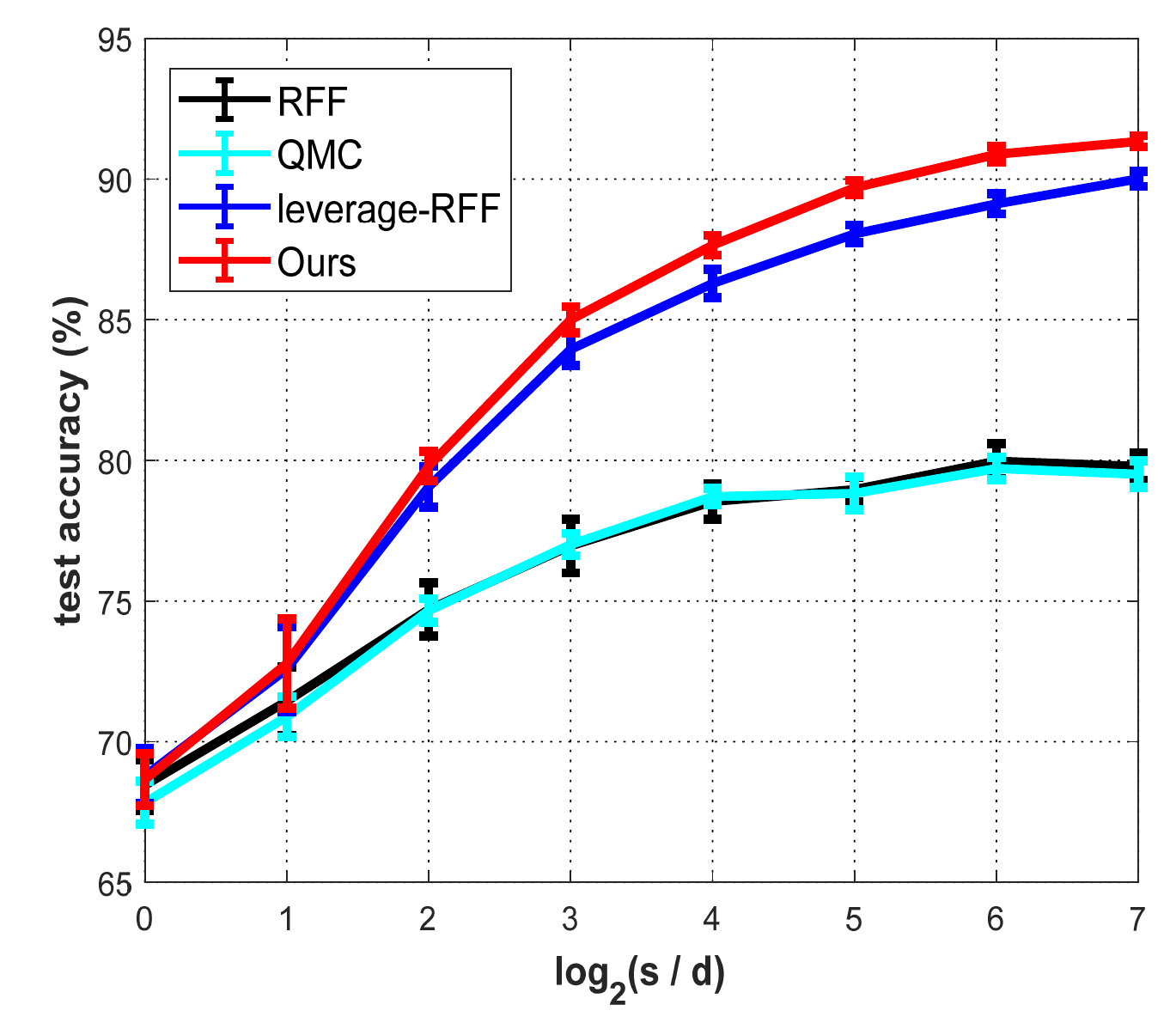}}
\subfigure[Time cost]{\label{stime}
\includegraphics[width=0.3\textwidth]{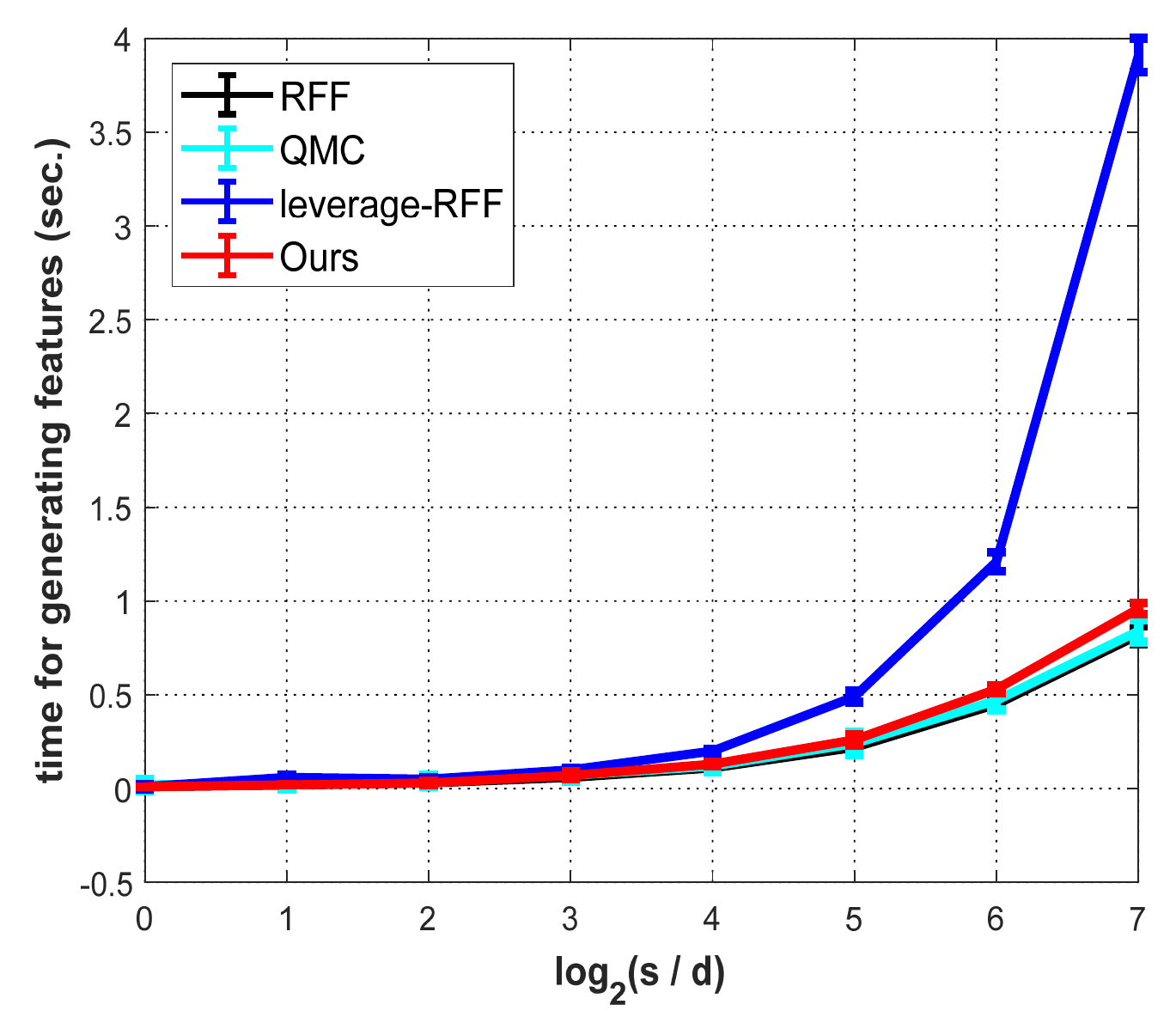}}
 \caption{Comparison of various algorithms on approximation error in (a), test accuracy in (b), and time cost for generating random features in (c) versus the number of random features $s$ on the \emph{EEG} dataset.}\label{tl1paramns}
 \end{figure*}

\section{Experiments}
\label{sec:exp}
In this section, we empirically evaluate the performance of our method equipped with KRR for classification tasks on several benchmark datasets.
All experiments are implemented in MATLAB and carried out on a PC with Intel$^\circledR$ i5-6500 CPU (3.20 GHz) and 16 GB RAM.
The source code of our implementation can be found in \url{http://www.lfhsgre.org}.

\subsection{Experimental settings}
We choose the popular shift-invariant Gaussian/RBF kernel for experimental validation, i.e., $k(\bm x, \bm x') = \exp(-\| \bm x - \bm x' \|^2/\sigma^2)$.
Following \cite{avron2017random}, we use a fixed bandwidth $\sigma=1$ in our experiments.
This is without loss of generality since we can rescale the points and adjust the bounding interval.
The regularization parameter $\lambda$ is tuned via 5-fold inner cross validation over a grid of $\{ 0.05,0.1,0.5,1 \}$.

\begin{table}
	\centering
	\scriptsize
	\caption{\footnotesize Dataset statistics.}
	\label{tablarge}
	\begin{threeparttable}
		\begin{tabular}{cccccccccccccc}
			\toprule[1.5pt]
			datasets & \#feature dimension  & \#traing examples & \#test examples\\
			\midrule[1pt]
			\emph{EEG} &14 &7,490 &7,490
			\\
			\hline
            \emph{cod-RNA} &8 &59,535 &157,413
			\\
			\hline
			\emph{covtype} &54 &290,506 &290,506
			\\
            \hline
            \emph{magic04} &10 & 9510 &9510 \\
			\bottomrule[1.5pt]
		\end{tabular}
	\end{threeparttable}
\end{table}

{\bf Datasets:}
We consider four classification datasets including \emph{EEG}, \emph{cod-RNA}, \emph{covtype} and \emph{magic04}; see Tab.~\ref{tablarge} for an overview of these datasets.
These datasets can be downloaded from \url{https://www.csie.ntu.edu.tw/~cjlin/libsvmtools/datasets/} or the UCI Machine Learning Repository\footnote{\url{https://archive.ics.uci.edu/ml/datasets.html.}}.
All datasets are normalized to $[0, 1]^d$ before the experiments.
We use the given training/test partitions on the \emph{cod-RNA} dataset.
For the other three datasets, we randomly pick half of the data for training and the rest for testing.
All  experiments are repeated 10 times and we report the average classification accuracy as well as the time cost for generating random features.

{\bf Compared methods:}
We compare the proposed surrogate leverage weighted sampling strategy with the following three random features based algorithms:
\begin{itemize}
	\item RFF \cite{rahimi2007random}: The feature mapping $\varphi_p(\bm x)$ is given by Eq.~\eqref{mapping}, in which the random features $\{ \bm w_i\}_{i=1}^s$ are sampled from $p(\bm w)$.
	\item QMC \cite{Avron2016Quasi}: The feature mapping $\varphi_p(\bm x)$ is also given by Eq.~\eqref{mapping}, but the random features $\{ \bm w_i\}_{i=1}^s$ are constructed by a deterministic scheme, e.g., a low-discrepancy Halton sequence.
	\item leverage-RFF \cite{li2019towards}: The data-dependent random features $\{ \bm w_i\}_{i=1}^s$ are sampled from $q^*(\bm w)$.
The kernel matrix in $q^*(\bm w)$ is approximated using random features pre-sampled from $p(\bm w)$.
\end{itemize}
In our method, we consider sampling from $\tilde{q}'(\bm w)$ in Algorithm~\ref{ago3} for simplicity.

\begin{table*}[!htb]
  \resizebox{1.98\columnwidth}{!}{
  \begin{threeparttable}
  \caption{Comparison results of various algorithms for varying $s$ in terms of classification accuracy (mean$\pm$std. deviation \%) and time cost for generating random features (mean$\pm$std. deviation sec.). The higher test accuracy means better. Notation ``$\bullet$" indicates that leverage-RFF and our method are significantly better than the other two baseline methods via paired t-test.}
  \label{Tabres}
    \begin{tabular}{ccccc|cccc}
    \toprule[1.5pt]
    &\multirow{2}{0.7cm}{Dataset}&\multirow{2}{0.2cm}{$s$} &RFF &QMC &leverage-RFF & Ours  \cr
     \cmidrule(lr){4-7}
     &&&Acc:\% (time:sec.) &Acc:\% (time:sec.)&Acc:\% (time:sec.) &Acc:\% (time:sec.)  \cr
    \midrule[1pt]
    &\multirow{8}{1.2cm}{\emph{EEG}}
    & $d$     & 68.45$\pm$0.89 (0.01$\pm$0.00) & 67.83$\pm$0.73 (0.02$\pm$0.03) & 68.78$\pm$0.97 (0.01$\pm$0.01) & 68.62$\pm$0.89 (0.01$\pm$0.00) \\
    &&$2d$    & 71.44$\pm$1.22 (0.02$\pm$0.00) & 70.89$\pm$0.72 (0.03$\pm$0.03) & 72.59$\pm$1.51 (0.03$\pm$0.01) & 72.72$\pm$1.35 (0.02$\pm$0.00) \\
     &&$4d$     & 74.70$\pm$0.94 (0.03$\pm$0.01) & 74.66$\pm$0.42 (0.04$\pm$0.03) & 79.06$\pm$0.73 (0.05$\pm$0.01)$\bullet$ & 79.72$\pm$0.58 (0.03$\pm$0.01)$\bullet$ \\
     &&$8d$     & 76.96$\pm$0.96 (0.06$\pm$0.02) & 77.01$\pm$0.40 (0.07$\pm$0.03) & 83.95$\pm$0.58 (0.10$\pm$0.01)$\bullet$ & 84.97$\pm$0.50 (0.07$\pm$0.02)$\bullet$ \\
     &&$16d$    & 78.54$\pm$0.63 (0.11$\pm$0.00) & 78.71$\pm$0.29 (0.12$\pm$0.03) & 86.29$\pm$0.50 (0.20$\pm$0.01)$\bullet$ & 87.23$\pm$0.41 (0.13$\pm$0.01)$\bullet$ \\
     &&$32d$    & 78.96$\pm$0.44 (0.22$\pm$0.02) & 78.83$\pm$0.59 (0.24$\pm$0.06) & 88.05$\pm$0.31 (0.49$\pm$0.03)$\bullet$ & 89.38$\pm$0.32 (0.26$\pm$0.03)$\bullet$ \\
     &&$64d$    & 79.97$\pm$0.62 (0.45$\pm$0.01) & 79.71$\pm$0.40 (0.47$\pm$0.05) & 89.12$\pm$0.36 (1.21$\pm$0.05)$\bullet$ & 90.36$\pm$0.31 (0.53$\pm$0.02)$\bullet$ \\
     &&$128d$   & 79.79$\pm$0.49 (0.82$\pm$0.05) & 79.51$\pm$0.47 (0.84$\pm$0.06)& 90.01$\pm$0.27 (3.91$\pm$0.09)$\bullet$ & 91.02$\pm$0.32 (0.96$\pm$0.03)$\bullet$ \\
     \cmidrule(lr){3-7}
    &\multirow{8}{1.5cm}{\emph{cod-RNA}} &$d$   & 87.02$\pm$0.29 (0.06$\pm$0.01) & 87.20$\pm$0.00 (0.07$\pm$0.03) & 88.62$\pm$0.92 (0.09$\pm$0.02) & 89.64$\pm$0.87 (0.07$\pm$0.01)$\bullet$ \\
    &&$2d$  & 87.12$\pm$0.19 (0.12$\pm$0.01) & 87.65$\pm$0.00 (0.16$\pm$0.02) & 90.42$\pm$1.15 (0.17$\pm$0.01)$\bullet$ & 90.12$\pm$0.95 (0.13$\pm$0.01)$\bullet$ \\
    &&$4d$  & 87.19$\pm$0.08 (0.24$\pm$0.01) & 87.44$\pm$0.00 (0.25$\pm$0.02) & 92.65$\pm$0.38 (0.35$\pm$0.02)$\bullet$ & 92.83$\pm$0.33 (0.27$\pm$0.01)$\bullet$ \\
    &&$8d$  & 87.27$\pm$0.11 (0.47$\pm$0.02) & 87.29$\pm$0.00 (0.49$\pm$0.02) & 93.41$\pm$0.07 (0.69$\pm$0.02)$\bullet$ & 93.49$\pm$0.15 (0.53$\pm$0.02)$\bullet$ \\
    &&$16d$ & 87.29$\pm$0.08 (0.91$\pm$0.02) & 87.30$\pm$0.00 (0.94$\pm$0.04) & 93.71$\pm$0.06 (1.39$\pm$0.05)$\bullet$ & 93.74$\pm$0.05 (0.99$\pm$0.02)$\bullet$ \\
    &&$32d$ & 87.27$\pm$0.05 (1.80$\pm$0.02) & 87.33$\pm$0.00 (1.77$\pm$0.01) & 93.76$\pm$0.02 (2.82$\pm$0.08)$\bullet$ & 93.71$\pm$0.07 (1.95$\pm$0.03)$\bullet$ \\
  &&$64d$ & 87.30$\pm$0.03 (3.48$\pm$0.15) &87.32$\pm$0.00 (3.54$\pm$0.10) & 93.73$\pm$0.03 (6.54$\pm$0.53)$\bullet$
 & 93.99$\pm$0.06 (4.05$\pm$0.08)$\bullet$ \\
    &&$128d$ & 87.30$\pm$0.03 (6.79$\pm$0.39)
& 87.32$\pm$0.00 (6.62$\pm$0.08)
 &93.66$\pm$0.03 (13.3$\pm$0.23)$\bullet$
&93.48$\pm$0.04 (7.78$\pm$0.09)$\bullet$
\\
    \cmidrule(lr){3-7}
    &\multirow{5}{1cm}{\emph{covtype}\tnote{1}} &
    $d$   & 73.70$\pm$0.79 (1.90$\pm$0.03) & 74.71$\pm$0.07 (1.88$\pm$0.11) & 73.99$\pm$0.85 (2.96$\pm$0.09) & 73.99$\pm$0.63 (2.00$\pm$0.05) \\
    &&$2d$  & 77.09$\pm$0.25 (3.31$\pm$0.21) & 77.04$\pm$0.06 (3.37$\pm$0.29) & 77.04$\pm$0.35 (5.25$\pm$0.09) & 77.02$\pm$0.29 (3.44$\pm$0.10) \\
    &&$4d$  & 79.10$\pm$0.13 (6.27$\pm$0.35) & 79.07$\pm$0.07 (6.12$\pm$0.17) & 79.18$\pm$0.17 (10.2$\pm$0.15) & 79.05$\pm$0.14 (6.58$\pm$0.19) \\
    &&$8d$  & 81.04$\pm$0.12 (12.3$\pm$0.71) & 80.90$\pm$0.05 (12.1$\pm$0.45) & 81.09$\pm$0.07 (21.1$\pm$0.72) & 80.79$\pm$0.11 (13.2$\pm$0.34) \\
    &&$16d$ & 82.42$\pm$0.10 (24.5$\pm$1.02) & 82.37$\pm$0.07 (24.3$\pm$1.56) & 82.90$\pm$0.12 (46.5$\pm$2.20) & 82.18$\pm$0.10 (28.6$\pm$1.58) \\
   \cmidrule(lr){3-7}
    &\multirow{8}{1.5cm}{\emph{magic04}} &$d$   & 73.62$\pm$0.68 (0.01$\pm$0.00) & 71.74$\pm$0.40 (0.02$\pm$0.04) & 73.62$\pm$0.68 (0.01$\pm$0.01) & 73.61$\pm$0.68 (0.01$\pm$0.00) \\
   && $2d$  & 75.89$\pm$0.80 (0.01$\pm$0.01) & 75.98$\pm$0.36 (0.03$\pm$0.03) & 75.91$\pm$0.77 (0.03$\pm$0.01) & 75.88$\pm$0.77 (0.02$\pm$0.00) \\
   && $4d$  & 77.78$\pm$0.45 (0.03$\pm$0.01) & 77.27$\pm$0.33 (0.04$\pm$0.03) & 77.78$\pm$0.45 (0.05$\pm$0.01) & 77.77$\pm$0.43 (0.03$\pm$0.00) \\
   && $8d$  & 78.97$\pm$0.34 (0.05$\pm$0.00) & 79.07$\pm$0.17 (0.07$\pm$0.03) & 79.15$\pm$0.40 (0.09$\pm$0.01) & 79.12$\pm$0.34 (0.06$\pm$0.01) \\
    &&$16d$ & 80.04$\pm$0.34 (0.10$\pm$0.01) & 79.95$\pm$0.37 (0.11$\pm$0.03) & 80.80$\pm$0.40 (0.19$\pm$0.01) & 80.74$\pm$0.42 (0.11$\pm$0.01) \\
    &&$32d$ & 80.61$\pm$0.43 (0.19$\pm$0.01) & 80.65$\pm$0.31 (0.21$\pm$0.04) & 82.00$\pm$0.32 (0.41$\pm$0.03)$\bullet$ & 82.02$\pm$0.32 (0.22$\pm$0.01)$\bullet$ \\
    &&$64d$ & 80.91$\pm$0.28 (0.38$\pm$0.03) & 80.85$\pm$0.27 (0.41$\pm$0.05) & 82.39$\pm$0.32 (0.93$\pm$0.05)$\bullet$ & 82.37$\pm$0.25 (0.44$\pm$0.03)$\bullet$ \\
    &&$128d$ & 81.10$\pm$0.37 (0.73$\pm$0.03) & 81.08$\pm$0.29 (0.76$\pm$0.04) & 82.59$\pm$0.29 (2.61$\pm$0.15)$\bullet$ & 82.55$\pm$0.55 (0.87$\pm$0.02)$\bullet$ \\
    \bottomrule[1.5pt]
    \end{tabular}
    \begin{tablenotes}
        \footnotesize
        \item[1] Due to the memory limit, we cannot conduct the experiment on the \emph{covtype} dataset when $s \geq 32d$.
\end{tablenotes}
    \end{threeparttable}}
\end{table*}

\subsection{Comparison results}
\subsubsection{High-level comparison:}
We compare the empirical performance of the aforementioned random features mapping based algorithms. In Fig.~\ref{tl1paramns}, we consider the  \emph{EEG} dataset and plot the relative kernel matrix approximation error, the test classification accuracy and the time cost for generating random features versus different values of $s$.
Note that since we cannot compute the kernel matrix $\bm K$ on the entire dataset, we randomly sample 1,000 datapoints to construct the feature matrix $\bm Z_s \bm Z_s^{\!\top}$, and then calculate the relative approximation error, i.e., $err=\frac{\| \bm K - \bm Z_s \bm Z_s^{\!\top} \|_2}{\| \bm K \|_2}$.

Fig.~\ref{sapp} shows the mean of the approximation errors across 10 trials (with one standard deviation denoted by error bars) under different random features dimensionality.
We find that QMC achieves the best approximation performance when compared to RFF, leverage-RFF, and our proposed method.
Fig.~\ref{sacc} shows the test classification accuracy. We find that as the number
of random features increases, leverage-RFF and our method significantly outperform RFF and QMC .

From the above experimental results, we find that, admittedly, QMC achieves lower approximation error to some extent, but it does not translate to better classification performance when compared to leverage-RFF and our method.
The reason may be that the original kernel derived by the point-wise distance might not be suitable, and the approximated kernel is not optimal for classification/regression tasks, as discussed in \cite{avron2017random,munkhoeva2018quadrature,zhang2019f}.
As the ultimate goal of kernel approximation is to achieve better prediction performance, in the sequel we omit the approximation performance of these methods.

In terms of time cost for generating random features, Fig.~\ref{stime} shows that leverage-RFF is quite time-consuming when the total number of random features is large.
In contrast, our algorithm achieves comparable computational efficiency with RFF and QMC.
These results demonstrate the superiority of our surrogate weighted sampling strategy, which reduces the time cost.
\subsubsection{Detailed comparisons:}
Tab.~\ref{Tabres} reports the detailed classification accuracy and time cost for generating random features of all algorithms on the four datasets.
Observe that by using a \emph{data-dependent} sampling strategy, leverage-RFF and our method achieve better classification accuracy than RFF and QMC on the \emph{EEG} and \emph{cod-RNA} dataset when the dimensionality of random features increases.
In particular, on the \emph{EEG} dataset, when $s$ ranges from $2d$ to $128d$, the test accuracy of leverage-RFF and our method is better than RFF and QMC by around 1\% to nearly 11\%.
On the \emph{cod-RNA} dataset, the performance of RFF and QMC is worse than our method by over 6\%  when $s\geq4d$.
On the \emph{covtype} dataset, all four methods achieve similar the classification accuracy.
Instead, on the \emph{magic04} dataset, our algorithm and leverage-RFF perform better than RFF and QMC on the final classification accuracy if more random features are considered.

In terms of computational efficiency on these four datasets, albeit \emph{data-dependent},
our method still takes about the similar time cost with the \emph{data-independent} RFF and QMC for generating random features.
Specifically, when compared to leverage-RFF, our method achieves a
substantial computational saving.


\section{Conclusion}
\label{sec:conclusion}

In this work, we have proposed an effective \emph{data-dependent} sampling strategy for generating fast random features for kernel approximation.
Our method can significantly improve the generalization performance while achieving the same time complexity when compared to the standard RFF.
Our theoretical results and experimental validation have demonstrated the superiority of our method when compared to other representative random Fourier features based algorithms on several classification benchmark datasets.

\section*{Acknowledgments}
This work was supported in part by the National Natural Science Foundation of China (No.61572315, 61876107, 61977046), in part by the National Key Research and Development Project (No. 2018AAA0100702), in part by National Science Foundation grants CCF-1657420 and CCF-1704828, in part by the European Research Council under the European Union's Horizon 2020 research and innovation program / ERC Advanced Grant E-DUALITY (787960). This paper reflects only the
authors' views and the Union is not liable for any use that may be made
of the contained information; Research Council KUL C14/18/068;
Flemish Government FWO project GOA4917N; Onderzoeksprogramma Artificiele
Intelligentie (AI) Vlaanderen programme.
Jie Yang and Xiaolin Huang are corresponding authors.


\end{document}